%% file: main-arxiv.tex
\documentclass[twocolumn]{article}
\PassOptionsToPackage{dvipsnames,table}{xcolor} 
\usepackage{preprint}
\SetBgContents{} 
\usepackage[utf8]{inputenc}
\usepackage[T1]{fontenc}
\usepackage[english]{babel}
\usepackage{csquotes}
\usepackage[final,
            tracking=true,
            kerning=true]{microtype}
\usepackage{xspace}

\usepackage{amssymb}
\usepackage{amsmath}
\usepackage{amsthm}
\usepackage{mathtools}
\usepackage{bbm}
\let\leq\leqslant
\let\geq\geqslant

\newtheorem{example}{Example}
\newtheorem{theorem}{Theorem}

\newtheorem{lemma}{Lemma}

\newtheorem{definition}{Definition}

\usepackage{graphicx}
\usepackage{booktabs}
\graphicspath{{./img/}}

\usepackage{nth}
\usepackage{siunitx}
\usepackage{enumitem}

\newcommand*{\titletext}{How to compare adversarial robustness of classifiers from a global perspective}

\usepackage[backend=biber,
            sorting=nyt,
            maxcitenames=2,
            maxbibnames=8]{biblatex}
\usepackage{varioref}
\PassOptionsToPackage{hyphens}{url}
\usepackage[hypertexnames=false,
            colorlinks=true,
            linkcolor=MidnightBlue,
            urlcolor=Maroon,
            citecolor=ForestGreen,
            pdfauthor={Jan Philip Göpfert},
            pdftitle={\titletext}]{hyperref}

\usepackage[all]{hypcap}
\usepackage[noabbrev,
            capitalize,
            nameinlink]{cleveref}

\usepackage[backgroundcolor=white,
            linecolor=lightgray,
            bordercolor=white,
            textsize=tiny]{todonotes}

\input{macros.tex}

\title{\titletext}

\usepackage[auth-sc]{authblk}

\author{Niklas Risse\textsuperscript{*}}
\author{Christina Göpfert\textsuperscript{*}}
\author{Jan Philip Göpfert\textsuperscript{*}}
\affil{Bielefeld University, Germany}

\newcommand\blfootnote[1]{{%
  \let\thempfn\relax%
  \footnotetext[0]{#1}%
}}

\begin{document}

\twocolumn[
\begin{@twocolumnfalse}
\maketitle
\begin{abstract}
\input{abstract.tex}
\end{abstract}
\vspace{0.5cm}
\end{@twocolumnfalse}
]
\blfootnote{\textsuperscript{*}equal contribution}
\input{content.tex}
\AtNextBibliography{\raggedright\small}
\printbibliography

\appendix
\input{appendix.tex}
\end{document}

%% file: macros.tex
\newcommand*{\eg}{eg.,\@\xspace}
\newcommand*{\ie}{i.\,e.\@\xspace}

\newcommand*{\wrt}{w.\,r.\,t.\@\xspace}
\newcommand*{\R}{\mathbb{R}}

\newcommand*{\1}{\mathbbm{1}}
\DeclarePairedDelimiterX{\norm}[1]{\lVert}{\rVert}{#1}

\DeclareMathOperator{\argmax}{argmax}
\DeclareMathOperator{\sgn}{sgn}

\newcommand*{\X}{\mathcal{X}}  
\newcommand*{\Y}{\mathcal{Y}}  
\renewcommand{\P}{P}  
\newcommand*{\dis}{d}  
\newcommand*{\rob}[2]{R_{#1}^{#2}} 
\newcommand*{\rr}{\mathcal{R}}
\newcommand*{\ddim}{m}

\newcommand*{\tmst}{\texttt{ST}\@\xspace}
\newcommand*{\tmat}{\texttt{AT}\@\xspace}
\newcommand*{\tmkw}{\texttt{KW}\@\xspace}
\newcommand*{\tmma}{\texttt{MMR\,+\,AT}\@\xspace}
\newcommand*{\tmmu}{\texttt{MMR-UNIV}\@\xspace}

\newcommand*{\dsmnist}{\texttt{MNIST}\@\xspace}
\newcommand*{\dsfmnist}{\texttt{FMNIST}\@\xspace}
\newcommand*{\dsgts}{\texttt{GTS}\@\xspace}
\newcommand*{\dscifar}{\texttt{CIFAR-10}\@\xspace}
\newcommand*{\dstinyimg}{\texttt{TINY-IMG}\@\xspace}
\newcommand*{\dshar}{\texttt{HAR}\@\xspace}

%% file: abstract.tex
Adversarial robustness of machine learning models has attracted considerable attention over recent years.
Adversarial attacks undermine the reliability of and trust in machine learning models, but the construction of more robust models hinges on a rigorous understanding of adversarial robustness as a property of a given model.
Point-wise measures for specific threat models are currently the most popular tool for comparing the robustness of classifiers and are used in most recent publications on adversarial robustness.
In this work, we use recently proposed robustness curves to show that point-wise measures fail to capture important global properties that are essential to reliably compare the robustness of different classifiers.
We introduce new ways in which robustness curves can be used to systematically uncover these properties and provide concrete recommendations for researchers and practitioners when assessing and comparing the robustness of trained models.
Furthermore, we characterize scale as a way to distinguish small and large perturbations, and relate it to inherent properties of data sets, demonstrating that robustness thresholds must be chosen accordingly.
We release code to reproduce all experiments presented in this paper, which includes a Python module to calculate robustness curves for arbitrary data sets and classifiers, supporting a number of frameworks, including TensorFlow, PyTorch and JAX.

%% file: content.tex
\section{Introduction} \label{sec: introduction}
Despite their astonishing success in a wide range of classification tasks, deep neural networks can be lead to incorrectly classify inputs altered with specially crafted adversarial perturbations~\autocite{Szegedy2014Intriguing,Goodfellow2014Explaining}.
These perturbations can be so small that they remain almost imperceptible to human observers~\autocite{Gopfert2020AdversarialAttacksHidden}.
Adversarial robustness describes a model's ability to behave correctly under such small perturbations crafted with the intent to mislead the model.
The study of adversarial robustness -- with its definitions, their implications, attacks, and defenses -- has attracted considerable research interest.
This is due to both the practical importance of trustworthy models as well as the intellectual interest in the differences between decisions of machine learning models and our human perception.
A crucial starting point for any such analysis is the definition of what exactly a small input perturbation is -- requiring (a) the choice of a \emph{distance function} to measure perturbation size, and (b) the choice of a particular \emph{scale} to distinguish small and large perturbations.
Together, these two choices determine a \emph{threat model} that defines exactly under which perturbations a model is required to be robust.

The most popular choice of distance function is the class of distances induced by $\ell_p$ norms~\autocite{Szegedy2014Intriguing,Goodfellow2014Explaining,carlini2019evaluating}, in particular $\ell_1, \ell_2$ and $\ell_\infty$, although other choices such as Wasserstein distance have been explored as well~\autocite{wang2019wasserstein}.
Regarding scale, the current default is to pick some perturbation threshold $\varepsilon$ without providing concrete reasons for the exact choice.
Analysis then focuses on the \emph{robust error} of the model, the proportion of test inputs for which the model behaves incorrectly under some perturbation up to size $\varepsilon$.
This means that the scale is defined as a binary distinction between small and large perturbations based on the perturbation threshold.
A set of canonical thresholds have emerged in the literature.
For example, in the publications referenced in this section, the \dsmnist data set is typically evaluated at a perturbation threshold $\varepsilon \in \{0.1, 0.3\}$ for the $\ell_\infty$ norm, while \dscifar is evaluated at $\varepsilon \in \{2/255, 4/255, 8/255\}$, stemming from the three 8-bit color channels used to represent images.

Based on these established threat models, researchers have developed specialized methods to minimize the robust error during training, which results in more robust models.
Popular approaches include specific data augmentation, sometimes used under the umbrella term adversarial training \autocite{guo2017countering,madry2018towards,carmon2019unlabeled,hendrycks2019using}, training under regularization that encourages large margins and smooth decision boundaries in the learned model \autocite{hein2017formal,wong2017provable,Croce2018ProvableRobustness,Croce2020Provable}, and post-hoc processing or randomized smoothing of predictions in a learned model \autocite{lecuyer2019certified,cohen2019certified}.

In order to show the superiority of a new method, robust accuracies of differently trained models are typically compared for a handful of threat models and data sets, \eg $\ell_\infty(\varepsilon = 0.1)$ and $\ell_2(\varepsilon = 0.3)$ for \dsmnist.
Out of \num{22}~publications on adversarial robustness published at NeurIPS~2019, ICLR~2020, and ICML~2020, \num{12}~publications contain results for only a single perturbation threshold.
In five publications, robust errors are calculated for at least two different perturbation thresholds, but still, only an arbitrary number of thresholds is considered.
Only in five out of the total \num{22}~publications do we find extensive considerations of different perturbation thresholds and the respective robust errors.
Out of these five, three are analyses of randomized smoothing, which naturally gives rise to certification radii~\autocite{li_certified_2019,carmon2019unlabeled,pinot_theoretical_2019}.
\Textcite{najafi_robustness_2019} follow a learning-theoretical motivation, which results in an error bound as a function of the perturbation threshold.
Only \textcite{maini_adversarial_2020} do not rely on randomization and still provide a complete, empirical analysis of robust error for varying perturbation thresholds\footnote{\raggedright
\emph{Single thresholds:} \autocite{mao2019metric,tramer_adversarial_2019,alayrac_are_2019,brendel_accurate_2019,qin_adversarial_2019,wang_improving_2020,song_robust_2020,Croce2020Provable,xie_intriguing_2020,rice_overfitting_2020,zhang_attacks_2020,singla_second-order_2020},
\emph{multiple thresholds:} \autocite{lee_tight_2019,mahloujifar_empirically_2019,hendrycks2019using,wong_fast_2020,boopathy_proper_2020},
\emph{full analysis:} \autocite{pinot_theoretical_2019,carmon2019unlabeled,li_certified_2019,najafi_robustness_2019,maini_adversarial_2020}.}.

\emph{Our contributions:}
In this work, we demonstrate that point-wise measures of $\ell_p$ robustness are not sufficient to reliably and meaningfully compare the robustness of different classifiers.
We show that, both in theory and practice, results of model comparisons based on point-wise measures may fail to generalize to threat models with even slightly larger or smaller $\varepsilon$ and that robustness curves avoid this pitfall by design.
Furthermore, we show that point-wise measures are insufficient to meaningfully compare the efficacy of different defense techniques when distance functions are varied, and that robustness curves, again, are able to reliably detect and visualize this property.
Finally, we analyze how scale depends on the underlying data space, choice of distance function, and distribution.
It is our belief that the continued use of single perturbation thresholds in the adversarial robustness literature is due to a lack of awareness of the shortcomings of these measures.
Based on our findings we suggest that robustness curves should become the standard tool when comparing adversarial robustness of classifiers, and that the perturbation threshold of threat models should be selected carefully in order to be meaningful, considering inherent characteristics of the data set.
We release code to reproduce all experiments presented in this paper\footnote{\raggedright The full code is available at \url{https://github.com/niklasrisse/how-to-compare-adversarial-robustness-of-classifiers-from-a-global-perspective}.}, which includes a Python module with an easily accessible interface (similar to Foolbox,~\textcite{rauber2017foolbox}) to calculate robustness curves for arbitrary data sets and classifiers. The module supports classifiers written in most of the popular machine learning frameworks, such as TensorFlow, PyTorch and JAX.

\section{Methods}\label{sec:methods}
An adversarial perturbation for a classifier $f$ and input-output pair $(x, y)$ is a small perturbation $\delta$ with $f(x + \delta) \neq y$.
Because the perturbation $\delta$ is small, it is assumed that the label $y$ would still be the correct prediction for $x + \delta$.
The resulting point $x + \delta$ is called an adversarial example.
The points vulnerable to adversarial perturbations are the points that are either already misclassified when unperturbed, or those that lie close to a decision boundary.

One tool to visualize and study the robustness behavior of a classifier are \emph{robustness curves}~\autocite{Gopfert2020AdversarialRobustnessCurves}.
A robustness curve captures the distribution of shortest distances between a set of points and the decision boundaries of a classifier:

\begin{definition}\label{Definition:robustness_curves}
    Given an input space $\X$ and label set $\Y$, distance function $\dis$ on $\X \times \X$, and classifier $f: \X \to \Y$.
    Assume $(x,y) \sim_{\text{i.i.d.}} \P$ for some distribution $\P$ on $\X \times \Y$.
    Then the $\dis$-\emph{robustness curve} for $f$ is the graph of the function
    \begin{equation*}
    \rob{\dis}{f}(\varepsilon) := \P\left(\{(x,y) \text{ s.t.
    } \exists\; x': \dis(x,x') \leq \varepsilon \wedge f(x') \neq y\}\right)
    \end{equation*}
\end{definition}

Since a model's robustness curve shows how data points are distributed in relation to the decision boundaries of the model, it allows us to take a step back from robustness regarding a specific perturbation threshold, and instead allows us to compare global robustness properties and their dependence on a given classifier, distribution and distance function.
To see why this is relevant, consider \cref{fig:toy}, which shows toy data along with two possible classifiers that perfectly separate the data.
For a perturbation threshold of $\varepsilon$, the blue classifier has robust error $0.5$, while the orange classifier is perfectly robust.
However, for a perturbation threshold of $2 \varepsilon$, the orange classifier has robust error $1$, while the blue classifier remains at $0.5$.
By freely choosing a single perturbation threshold for comparison, it is therefore possible to make either classifier appear to be much better than the other, and no single threshold can capture the whole picture.
In fact, for any two disjoint sets of perturbation thresholds, it is possible to construct a data distribution and two classifiers $f$, $f'$, such that the robust error of $f$ is lower than that of $f'$ for all perturbation thresholds in the first set, and that of $f'$ is lower than that of $f$ for all perturbation thresholds in the second set.
See \cref{app:robustness curve intersections} for a constructive proof. 

\begin{figure*}[tbp]
    \centering
    \raisebox{-0.5\height}{
        \def\svgwidth{3.35835cm}
        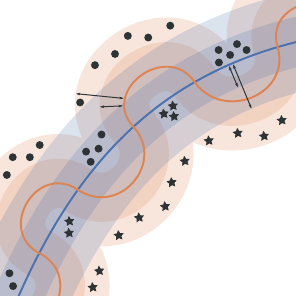
    }
    \hspace{1cm}
    \raisebox{-0.5\height}{
        \includegraphics{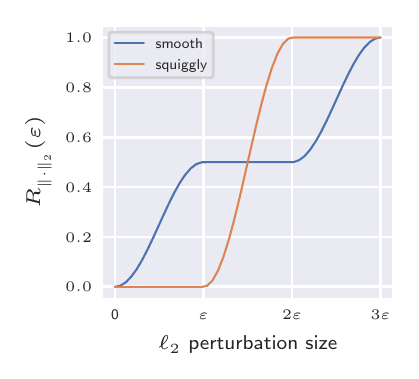}
    }
    \caption{Excerpt of a toy data set with two decision boundaries (left) and
        respective robustness curves (right).
        The data is separated perfectly by one
        smooth boundary (blue robustness curve), and one squiggly boundary (orange
        robustness curve).
        We indicate margins around the boundaries at distances
        \(\varepsilon\) and \(2 \varepsilon\). Selecting a single perturbation threshold is not sufficient to decide which classifier is more robust.}
    \label{fig:toy}
\end{figure*}

\section{Experiments} \label{sec: experiments}
In the following, we empirically evaluate the robustness of a number of recently published models, and demonstrate that the weaknesses of point-wise measures described above are not limited to toy examples, but occur for real-world data and models.

\subsection{Experimental Setup}\label{sec:experimental setup}
We evaluate and compare the robustness of models obtained using the following training methods:
\begin{enumerate}[nosep]
    \item Standard training (\tmst), \ie, training without specific robustness considerations.
    \item Adversarial training (\tmat)~\autocite{madry2018towards}.
    \item Training with robust loss (\tmkw)~\autocite{wong2017provable}.
    \item Maximum margin regularization for a single $\ell_p$ norm together with adversarial training (\tmma)~\autocite{Croce2018ProvableRobustness}.
    \item Maximum margin regularization simultaneously for $\ell_\infty$ and $\ell_1$ margins (\tmmu)~\autocite{Croce2020Provable}.
\end{enumerate}
Together with each training method, we state the threat model the trained model is optimized to defend against, \eg $\ell_\infty(\varepsilon = 0.1)$ for perturbations in $\ell_\infty$ norm with perturbation threshold $\varepsilon = 0.1$, if any.
The trained models are those made publicly available by \textcite{Croce2018ProvableRobustness}\footnote{\raggedright The models trained with \tmst, \tmkw, \tmat and \tmma are avaible at \url{www.github.com/max-andr/provable-robustness-max-linear-regions}.} and \textcite{Croce2020Provable}\footnote{\raggedright The models trained with \tmmu are avaible at \url{www.github.com/fra31/mmr-universal}.}.
The network architecture is a convolutional network with two convolutional layers, two fully connected layers and ReLU activation functions.
The evaluation is based on six real-world datasets: \dsmnist, Fashion-\dsmnist (\dsfmnist)~\autocite{xiao2017fashionmnist}, German Traffic Signs (\dsgts)~\autocite{Houben-IJCNN-2013}, \dscifar~\autocite{Krizhevsky09learningmultiple}, Tiny-Imagenet-200~(\dstinyimg)~\autocite{tinyimagenet}, and Human Activity Recognition~(\dshar) \autocite{hardataset}.
For specifics on model training (hyperparameters, architecture details), refer to \cref{app:experimental details}.
Models are generally trained on the full training set for the corresponding data set, and robustness curves evaluated on the full test set, unless stated otherwise.

For complex models, calculating the exact distance of a point to the closest decision boundary, and thus estimating the true robustness curve, is computationally very intensive, if not intractable.
Therefore we bound the true robustness curve from below using strong adversarial attacks, which is consistent with the literature on empirical evaluation of adversarial robustness and also applicable to many different types of classifiers.
We base our selection of attacks on the recommendations by \textcite{carlini2019evaluating}.
Specifically, we use the $\ell_2$-attack proposed by~\autocite{Carlini2017Towards}  for $\ell_2$ robustness curves and PGD~\autocite{madry2018towards} for $\ell_\infty$ robustness curves.
For both attacks, we use the implementations of Foolbox \autocite{rauber2017foolbox}.
See \cref{app:experimental details} for information on adversarial attack hyperparameters.
In the following, \enquote{robustness curve} refers to this empirical approximation of the true robustness curve.

\subsection{The weaknesses of point-wise measures}\label{sec:weaknesses point-wise measures}
Point-wise measures are used to quantify robustness of classifiers by measuring the robust test error for a specific distance function and a perturbation threshold (\eg $\ell_\infty(\varepsilon = 4/255)$).
In \cref{table:point wise measures} we show three point-wise measures to compare the robustness of five different classifiers on \dscifar.
If we compare the robustness of the four robust training methods (latter four columns of the table) based on the first point-wise threat model $\ell_\infty(\varepsilon = 1/255)$ (first row of the table), we can see that the classifier trained with \tmat is the most robust, followed by \tmma, followed by \tmkw, and \tmmu results in the least robust classifier.
However, if we increase the $\varepsilon$ of our threat model to $\varepsilon = 4/255$ (second row of the table), \tmkw is more robust than \tmat.
For a even larger $\varepsilon$ (third row of the table), we would conclude that \tmmu is preferable over \tmat, and that \tmat results in the least robust classifier.
All three statements are true for the particular perturbation threshold ($\varepsilon$), and the magnitude of all perturbation thresholds is reasonable: publications on adversarial robustness typically evaluate \dscifar on perturbation thresholds $\leq 10/255$ for $\ell_\infty$ perturbations.
Meaningful conclusions on the robustness of the classifiers relative to each other can not be made without taking all possible $\varepsilon$ into account.
In other words, a global perspective is needed.

\begin{table*}[bt]
    \centering
    \caption{Three point-wise measures for different threat models.
        All threat models use the $\ell_\infty$ distance function, but differ in choice of perturbation threshold (denoted by $\varepsilon$).
        Each row contains the robust test errors for one point-wise measure.
        Each column contains the robust test errors for one model, trained with a specific training method (marked by
        column title).
        The lower the number, the better the robustness for the specific threat model. Each point-wise measure results in a different relative ordering of the classifiers based on the errors.
        The  order is visualized by different tones of gray in the background of the cells. \newline}
    \resizebox{
        \ifdim\width>\linewidth
        \linewidth
        \else
        \width
        \fi
    }{!}{
        \small
        \begin{tabular}{cS[table-format=2.2]S[table-format=2.2]S[table-format=2.2]S[table-format=2.2]S[table-format=2.2]S[table-format=2.2]}
            \toprule   
            $\varepsilon$ & {\tmst}                  & {\tmat}                  & {\tmkw}                  &  {\tmma}                 & {\tmmu}                  \\
            \midrule
            1/255         & \cellcolor{gray!10} 0.60 & \cellcolor{gray!50} 0.38 & \cellcolor{gray!30} 0.43 & \cellcolor{gray!40} 0.42 & \cellcolor{gray!20} 0.54 \\
            4/255         & \cellcolor{gray!10} 0.99 & \cellcolor{gray!30} 0.68 & \cellcolor{gray!50} 0.57 & \cellcolor{gray!40} 0.63 & \cellcolor{gray!20} 0.74 \\
            8/255         & \cellcolor{gray!10} 1.00 & \cellcolor{gray!20} 0.92 & \cellcolor{gray!50} 0.73 & \cellcolor{gray!40} 0.84 & \cellcolor{gray!30} 0.91 \\
            \bottomrule
        \end{tabular}
    }
    \label{table:point wise measures}
\end{table*}

\subsubsection{A global perspective}

\Cref{fig:rc_crossings} shows the robustness of different classifiers for the $\ell_\infty$ (right plot) and $\ell_2$ (left plot) distance functions from a global perspective using robustness curves.
The plot reveals why the three point-wise measures (marked by vertical black dashed lines in the left plot) lead to different results in the relative ranking of robustness of the classifiers.
Both for the classifiers trained to be robust against attacks in $\ell_\infty$ distance (left plot) and $\ell_2$ distance (right plot), we can observe multiple intersections of robustness curves, corresponding to changes in the relative ranking of the robustness of the compared classifiers.
The robustness curves allow us to reliably compare the robustness of classifiers for all possible perturbation thresholds. Furthermore, the curves clearly show the perturbation threshold intervals with strong and weak robustness for each classifier, and are not biased by an arbitrarily chosen perturbation threshold.

\begin{figure*}[tbp]
    \centering
    \includegraphics{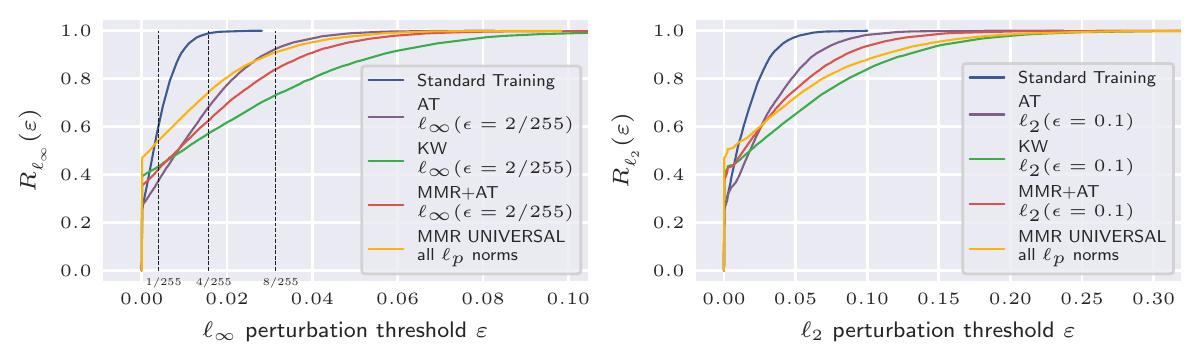}
    \caption{$\ell_\infty$ robustness curves (left plot) and $\ell_2$ robustness curves (right plot) resulting from different training methods (indicated by label), optimized for different threat models (indicated by
        label). The dashed vertical lines visualize the three point-wise measures from \cref{table:point wise measures}.
        The models are trained and evaluated on the full training-/test sets of \dscifar. The curves allow us to reliably compare the robustness of the classifiers, unbiased by choice of perturbation threshold.}
    \label{fig:rc_crossings}
\end{figure*}

\subsubsection{Overfitting to specific perturbation thresholds}
In addition to the problem of robustness curve intersection, relying on point-wise robustness measures to evaluate adversarial robustness is prone to overfitting when designing training procedures.
\Cref{fig:rc_sup_multiple_datasets} shows $\ell_\infty$ robustness curves for \tmma with $\ell_\infty$ threat model as provided by \textcite{Croce2018ProvableRobustness}.
The models trained on \dsmnist and \dsfmnist both show a change in slope, which could be a sign of overfitting to the specific threat models for which the classifiers were optimized for, since the change of slope occurs approximately at the chosen perturbation threshold $\varepsilon$.
This showcases a potential problem with the use of point-wise measures during training.
The binary separation of \enquote{small} and \enquote{large} perturbations based on the perturbation threshold is not sufficient to capture the intricacies of human perception under perturbations, but a simplification based on the idea that perturbations below the perturbation threshold should almost certainly not lead to a change in classification.
If a training procedure moves decision boundaries so that data points lie just beyond this threshold, it may achieve a low robust error, without furthering the actual goals of adversarial robustness research.
Using robustness curves for evaluation cannot prevent this effect, but can be used to detect it.

\begin{figure*}[tbp]
    \centering
    \includegraphics{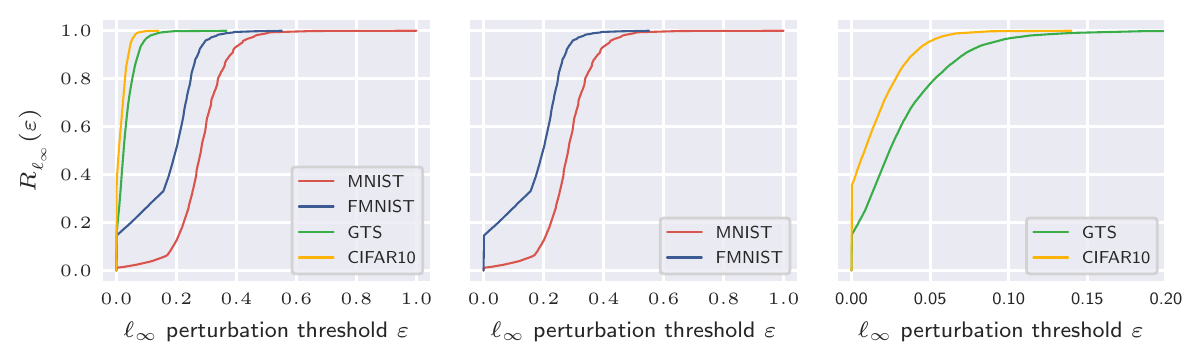}
    \caption{$\ell_\infty$ robustness curves for multiple data sets.
        Each curve
        is calculated for a different model and a different test data set.
        The data
        sets are indicated by the labels.
        The models are trained with \tmma, Threat
        Models: \dsmnist: $\ell_\infty(\varepsilon = 0.1)$, \dsfmnist:
        $\ell_\infty(\varepsilon = 0.1)$, \dsgts: $\ell_\infty(\varepsilon =
        4/255)$, \dscifar: $\ell_\infty(\varepsilon = 2/255)$. The curves for \dsmnist and \dsfmnist both show a change in slope, which can not be captured with point-wise measures and could be a sign of overfitting to the specific threat models for which the classifiers were optimized for. }
    \label{fig:rc_sup_multiple_datasets}
\end{figure*}

\subsubsection{Transfer of robustness across distance functions} \label{sec:robustness:transfer:threat}
In the following, we analyze to which extent properties of robustness curves transfer across different choices of distance functions.
If properties transfer, it may not be necessary to individually analyze robustness for each distance function.

In \cref{fig:rc_transfer} we compare the robustness of different models for the $\ell_\infty$ (left plot) and $\ell_2$ (right plot) distance functions.
The difference to \cref{fig:rc_crossings} is that the models (indicated by colour) are the same models in the left plot and in the right plot.
We find that for \tmma, the $\ell_\infty$ threat model leads to better robustness than the $\ell_2$ threat model \emph{both} for $\ell_\infty$ \emph{and} $\ell_2$ robustness curves.
In fact, \tmma with the $\ell_\infty$ threat model even leads to better $\ell_\infty$ and $\ell_2$ robustness curves than \tmmu, which is specifically designed to improve robustness for all $\ell_p$ norms.
Overall, the plots are visually similar.
However, since both plots contain multiple robustness curve intersections, the ranking of methods remains sensitive to the choice of perturbation threshold.
For example, a  perturbation threshold of $\varepsilon=3/255$ (vertical black dashed line) for the $\ell_\infty$ distance function (left subplot) shows that the classifier trained with \tmma ($\ell_2(\varepsilon=0.1)$) is approximately as robust as the classifier trained with \tmmu.
The same perturbation threshold for the $\ell_2$ distance function (right subplot) shows that the classifier trained with \tmma is more robust than the classifier trained with \tmmu for $\ell_2$ threat models.
Using typical perturbation thresholds from the literature for each distance function does not alleviate this issue:
At perturbation threshold $\varepsilon = 2 / 255$ for $\ell_\infty$ distance, the classifier trained with \tmma ($\ell_2(\varepsilon=0.1)$) is more robust than the one trained with \tmmu, while at perturbation threshold $\varepsilon = 0.1$ for $\ell_2$ distance, the opposite is true.
This shows that even when robustness curves across various distance functions are qualitatively similar, this may be obscured by the choice of threat model(s) to compare on.

We also emphasize that in general, robustness curves across various distance functions may be qualitatively \emph{dis}similar. In particular:
\begin{enumerate}[nosep]
    \item\label{item:robustness curve shape} For linear classifiers, the \emph{shape} of a robustness curve is identical for distances induced by different $\ell_p$ norms. This follows from \cref{thm:robustness curve shape} in \cref{app:robustness curves shape}, which is an extension of a weaker result in \textcite{Gopfert2020AdversarialRobustnessCurves}.
    For non-linear classifiers, different $\ell_p$ norms may induce different robustness curve shapes. See \textcite{Gopfert2020AdversarialRobustnessCurves} for an example.
    \item\label{item:robustness curve intersection} Even for linear classifiers, robustness curve \emph{intersections} do not transfer between distances induced by different $\ell_p$ norms. That is, for two linear classifiers, there may exist $p, p'$ such that the robustness curves for the $\ell_p$ distance intersect, but not the robustness curves for the $\ell_{p'}$ distance. See \cref{app:robustness curve intersections} for an example.
\end{enumerate}

\begin{figure*}[tbp]
    \centering
    \includegraphics{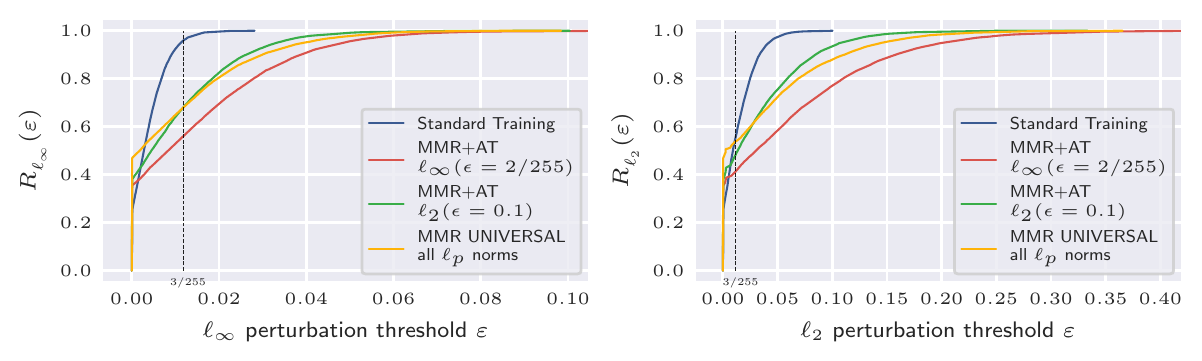}
    \caption{$\ell_\infty$ robustness curves (left plot) and $\ell_2$ robustness curves (right plot) resulting from different training methods (indicated by color and label), optimized for different threat models (indicated by
        label).
        The models are trained and evaluated on the full training-/test sets of \dscifar. The curves allow us to reliably compare the transfer of robustness of the classifiers across distance functions, unbiased by choice of threat model.}
    \label{fig:rc_transfer}
\end{figure*}

\subsection{On the relationship between scale and data}\label{sec:scale analysis for different data sets}
As the previous sections show, robustness curves can be used to reveal properties of robust models that may be obscured by point-wise measures.
However, some concept of scale, that is, some way to judge whether a perturbation is small or large, remains necessary.
\emph{Especially} when robustness curves intersect, it is crucial to be able to judge how critical it is for a model to be stable under the given perturbations.
For many pairs of distance function and data set, canonical perturbation thresholds have emerged in the literature, but to the best of our knowledge, no reasons for these choices are given.

Since the assumption behind adversarial examples is that small perturbations should not affect classification behavior, the question of scale cannot be answered independently of the data distribution.
In order to understand how to interpret different perturbation sizes, it can be helpful to understand how strongly the data point would need to be perturbed to \emph{actually} change the \emph{correct} classification.
We call this the \emph{inter-class distance} and analyze the distribution of inter-class distances for several popular data sets.

In \cref{fig:minimum_distances_3} we compare the inter-class distance distributions in $\ell_\infty$, $\ell_2$, and $\ell_1$ norm for all data sets considered in this work.
We observe that for the $\ell_1$ and $\ell_2$ norms, the shape of the curves is similar across data sets, but their extent is determined by the dimensionality of the data space.
In the $\ell_\infty$ norm, vastly different curves emerge for the different data sets.
We hypothesize that, because the inter-class distance distributions vary more strongly for $\ell_\infty$ distances than for $\ell_1$ distances, the results of robustifying a model \wrt $\ell_\infty$ distances may depend more strongly on the underlying data distribution than the results of robustifying \wrt $\ell_1$ distances.
This is an interesting avenue for future work.

When we look at the smallest inter-class distances in the $\ell_\infty$ norm (where all distances lie in the interval $[0, 1]$), we can make several observations.
Because the smallest inter-class distance for \dsmnist in the $\ell_\infty$ norm is around $0.9$, we can see that transforming an input from one class to one from a different class almost always requires completely flipping at least one pixel from almost-black to almost-white or vice versa.
For the other datasets, the inter-class distance distributions are more spread out than the inter-class distance distribution of \dsmnist.
We observe that for \dscifar with $\ell_\infty$ perturbations of size $\geq 0.25$, it becomes possible to transform samples from different classes into each other, so starting from this threshold, any classifier must necessarily trade off between accuracy and robustness.
The shapes of the curves and the threshold from which any classifier must necessarily trade of between accuracy and robustness differ strongly between data sets -- refer to \cref{table:minimum_distances} for exact values for the threshold.

\begin{figure*}[tbp]
    \centering
    \includegraphics{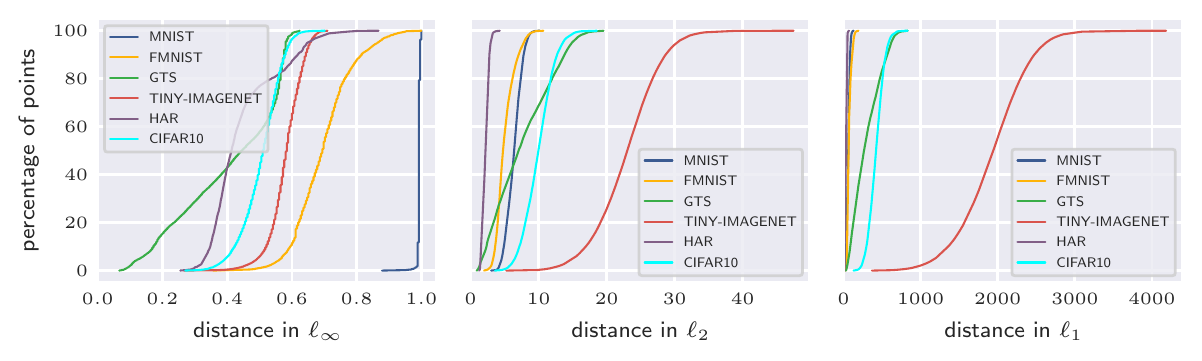}
    \caption{Minimum inter-class distances of all data sets considered in this
        work, measured in $\ell_\infty$ (left), $\ell_2$ (middle), and $\ell_1$
        (right) norm. See \cref{table:minimum_distances} for size and dimensionality. The shapes of the curves and the threshold from which any classifier must necessarily trade of between accuracy and robustness differ strongly between data sets.}
    \label{fig:minimum_distances_3}
\end{figure*}

\begin{table*}[bt]
    \rowcolors{2}{gray!12}{white}
    \centering
    \caption{Smallest and largest inter-class distances for subsets of several data
        sets, measured in $l_\infty$, $l_2$, and $l_1$ norm, together with basic
        contextual information about the data sets.
        All data has been been normalized to
        lie within the interval $[0,1]$, and duplicates and corrupted data points have
        been removed.
        Apart from \dshar, all data sets contain images -- the dimensionality
        reported specifies their sizes and number of channels.\newline}
    \resizebox{
        \ifdim\width>\linewidth
        \linewidth
        \else
        \width
        \fi
    }{!}{
        \begin{tabular}{lS[table-format=5.0]S[table-format=3.0]cS[table-format=1.2]S[table-format=2.2]S[table-format=3.2]S[table-format=2.2]S[table-format=2.2]S[table-format=4.2]}
            \toprule   
            & & & & \multicolumn{6}{c}{Inter-class Distance} \\
            & & & &  \multicolumn{3}{c}{{Smallest}} & \multicolumn{3}{c}{{Largest}} \\
            Dataset & {Samples} & {Classes} &  {Dimensionality} &  {$l_\infty$} & {$l_2$} & {$l_1$} & {$l_\infty$} & {$l_2$} & {$l_1$} \\
            \midrule
            \dsmnist    & 10000 &  10 & $28 \times 28 \times 1$ & 0.88 &  3.03 &  19.16 & 1.00 & 10.18 &  132.38 \\
            \dstinyimg  & 98139 & 200 & $64 \times 64 \times 3$ & 0.27 &  5.24 & 369.29 & 0.71 & 47.49 & 4184.37 \\
            \dsfmnist   & 10000 &  10 & $28 \times 28 \times 1$ & 0.36 &  2.00 &  24.87 & 1.00 & 10.70 &  194.29 \\
            \dsgts      & 10000 &  43 & $32 \times 32 \times 3$ & 0.07 &  0.90 &  31.46 & 0.62 & 19.54 &  833.22 \\
            \dscifar    & 10000 &  10 & $32 \times 32 \times 3$ & 0.27 &  3.61 & 130.77 & 0.70 & 18.57 &  831.44 \\
            \dshar      &  2947 &   6 &                   $561$ & 0.26 &  1.26 &  12.95 & 0.87 &  4.29 &   73.19 \\
            \bottomrule
        \end{tabular}
    }
    \label{table:minimum_distances}
\end{table*}

In \cref{table:minimum_distances}, we summarize the smallest and largest inter-class distances in different norms together with additional information about the size, number of classes, and dimensionality of the all the data sets we consider in this work.
The values correspond directly to \cref{fig:minimum_distances_3}, but even in this simplified view, we can quickly make out key differences between the data sets.
Compare, for example, \dsmnist and \dsgts: While it appears reasonable to expect $\ell_\infty$ robustness of $0.3$ for \dsmnist, the same threshold for \dsgts is not possible.
Relating \cref{table:minimum_distances} and \cref{fig:rc_sup_multiple_datasets}, we find entirely plausible the strong robustness results for \dsmnist, and the small perturbation threshold for \dsgts.
Based on inter-class distances we also expect less $\ell_\infty$ robustness for \dscifar than for \dsfmnist, but not as seen in \cref{fig:rc_sup_multiple_datasets}.
In any case, it is safe to say that, when judging the robustness of a model by a certain threshold, that number must be set with respect to the distribution the model operates on.

Overall, the strong dependence of robustness curves on the data set and the chosen norm, emphasizes the necessity of informed and conscious decisions regarding robustness thresholds.
We provide an easily accessible reference in the form of \cref{table:minimum_distances}, that should prove useful while judging scales in a threat model.

\section{Discussion} \label{sec: discussion}
We have demonstrated that comparisons of robustness of different classifiers using point-wise measures can be heavily biased by the choice of perturbation threshold and distance function of the threat model,
and that conclusions about rankings of classifiers with regards to their robustness based on point-wise measures therefore only provide a narrow view of the actual robustness behavior of the classifiers.
Further, we have demonstrated different ways of using robustness curves to overcome the shortcomings of point-wise measures, and therefore recommend using them as the standard tool for comparing the robustness of classifiers.
Finally, we have demonstrated how suitable perturbation thresholds necessarily depend on the data they pertain to.

It is our hope that practitioners and researchers alike will use the methodology proposed in this work, especially when developing and comparing adversarial defenses, and carefully motivate any concrete threat models they might choose, taking into account all available context.

\emph{Limitations:}
Computing approximate robustness curves for state-of-the-art classifiers and large data sets is computationally very intensive, due to the need of computing approximate minimal adversarial perturbations with strong adversarial attacks.
Developing adversarial attacks which are both strong and fast is an ongoing challenge in the field of adversarial robustness.
Another limitation of our work is the focus on a small group of distance functions (mainly $\ell_\infty$ and $\ell_2$ norms).
Even though it does intuitively make sense that models should at least be robust against these types of perturbations, a more general evaluation able to consider more distance functions simultaneously could be advantageous.

%% file: img/fig_toy.pdf_tex
\begingroup%
  \makeatletter%
  \providecommand\color[2][]{%
    \errmessage{(Inkscape) Color is used for the text in Inkscape, but the package 'color.sty' is not loaded}%
    \renewcommand\color[2][]{}%
  }%
  \providecommand\transparent[1]{%
    \errmessage{(Inkscape) Transparency is used (non-zero) for the text in Inkscape, but the package 'transparent.sty' is not loaded}%
    \renewcommand\transparent[1]{}%
  }%
  \providecommand\rotatebox[2]{#2}%
  \newcommand*\fsize{\dimexpr\f@size pt\relax}%
  \newcommand*\lineheight[1]{\fontsize{\fsize}{#1\fsize}\selectfont}%
  \ifx\svgwidth\undefined%
    \setlength{\unitlength}{85.03937008bp}%
    \ifx\svgscale\undefined%
      \relax%
    \else%
      \setlength{\unitlength}{\unitlength * \real{\svgscale}}%
    \fi%
  \else%
    \setlength{\unitlength}{\svgwidth}%
  \fi%
  \global\let\svgwidth\undefined%
  \global\let\svgscale\undefined%
  \makeatother%
  \begin{picture}(1,1)%
    \lineheight{1}%
    \setlength\tabcolsep{0pt}%
    \put(0,0){\includegraphics[width=\unitlength,page=1]{fig_toy.pdf}}%
    \put(0.42651749,0.94212508){\color[rgb]{0.18039216,0.20392157,0.21176471}\makebox(0,0)[t]{\smash{\begin{tabular}[t]{c}$2\varepsilon$\end{tabular}}}}%
    \put(0.7260821,0.40981325){\color[rgb]{0.18039216,0.20392157,0.21176471}\makebox(0,0)[t]{\smash{\begin{tabular}[t]{c}$\varepsilon$\end{tabular}}}}%
    \put(0,0){\includegraphics[width=\unitlength,page=2]{fig_toy.pdf}}%
  \end{picture}%
\endgroup%

%% file: appendix.tex
\setcounter{theorem}{0}
\setcounter{proposition}{0}
\section{Robustness curves with arbitrary intersections}\label{app:robustness curve intersections}

\begin{theorem}
    Let $T_1, T_2 \subset \R^{>0}$ be two disjoint finite sets.
    Then there exists a distribution $\P$ on $\R \times \{0,1\}$ and two classifiers $c_1, c_2 : \R\to \{0,1\}$ such that $\rob{|\cdot|}{c_1}(t) < \rob{|\cdot|}{c_2}(t)$ for all $t \in T_1$ and $\rob{|\cdot|}{c_1}(t) > \rob{|\cdot|}{c_2}(t)$ for all $t \in T_2$.
\end{theorem}

\begin{proof}
    Without loss of generality, assume that $T_1 = \{t_1, \dots, t_n\}$ and $T_2 = \{t_1', \dots, t_n'\}$ with $t_i < t_i' < t_{i+1}$ for $i \in \{1,\dots, n\}$.
    We will construct $c_1, c_2$ such that the robustness curves $\rob{|\cdot|}{c_1}(\cdot), \rob{|\cdot|}{c_2}(\cdot)$ intersect at exactly the points $(t_i + t_i')/2$ and $(t_i + t_{i+1}')/2$ on the interval $(t_1, t_n']$. 
    Let $d = t_n'$ and
    \begin{equation*}
        \P\left( -d - \frac{t_i + t_{i+1}'}{2}, 0 \right) = \P\left(d + \frac{t_i + t_i'}{2}, 1 \right) = \frac{2}{4n+1}
    \end{equation*}
    and
    \begin{equation*}
        P\left(-d-\frac{t_1}{2}, 0\right) = \frac{1}{4n+1}\,.
    \end{equation*}
    Let $c_1(x) = \1_{x \geq -d}$ and $c_2(x) = \1_{x \geq d}$.
    Both classifiers have perfect accuracy on $\P$, meaning that $\rob{|\cdot|}{c_i}(0) = 0$.
    The closest point to the decision boundary of $c_1$ is $-d-\frac{t_1}{2}$ with weight $\frac{1}{4n+1}$, so $\rob{|\cdot|}{c_1}(\frac{t_1}{2}) = \frac{1}{4n+1}$.
    The second-closest point is $-d - \frac{t_1 + t_{2}'}{2}$ with weight $\frac{2}{4n+1}$, so $\rob{|\cdot|}{c_1}(\frac{t_1 + t_2'}{2}) = \frac{3}{4n+1}$, and so on.
    Meanwhile, the closest point to the decision boundary of $c_2$ is $d+\frac{t_1 + t_1'}{2}$ with weight $\frac{2}{4n+1}$, so $\rob{|\cdot|}{c_2}(\frac{t_1+t_1'}{2}) = \frac{2}{4n+1}$, the second-closest point is $d  \frac{t_2 + t_{2}'}{2}$ with weight $\frac{2}{4n+1}$, so $\rob{|\cdot|}{c_2}(\frac{t_2 + t_2'}{2}) = \frac{4}{4n+1}$, and so on. 
\end{proof}

\begin{example}
    To see that robustness curve intersections do not transfer between different $\ell_p$ norms, consider the example in \cref{fig:toy_intersect}.
    The blue and orange linear classifiers both perfectly separate the displayed data. The $\ell_\infty$ robustness curves of the classifiers do not intersect, meaning that the robust error of the blue classifier is always better than that of the orange classifier.
    In $\ell_2$ distance, the robustness curves intersect, so that there is a range of perturbation sizes where the orange classifier has better robust error than the blue classifier.
    \begin{figure*}[tbp]
    \centering
\includegraphics{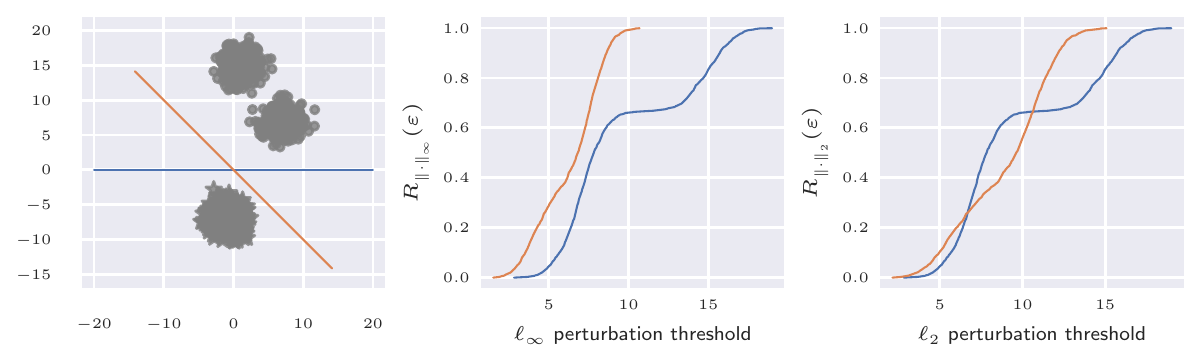}
        \caption{Example of a data distribution and two linear classifiers such that the $\ell_2$ robustness curves intersect, but not the $\ell_\infty$ robustness curves.}
        \label{fig:toy_intersect}
    \end{figure*}
\end{example}

\section{Robustness curve dependence of shape on distance function}\label{app:robustness curves shape}

\begin{theorem}\label{thm:robustness curve shape}
    Let $f(x) = \sgn(w^Tx+b)$ be a linear classifier.
    Then the \emph{shape} of the robustness curve for $f$ regarding an $\ell_p$ norm-induced distance does not depend on the choice of $p$.
    It holds that
    \begin{equation}
    \rob{\ell_{p_1}}{f}(\varepsilon) = \rob{\ell_{p_2}}{f}(c \cdot \varepsilon) \quad \forall \; \varepsilon \text{ for } c = \frac{ \|w\| _ { q_1 } }{ \|w\| _ {
            q_2 } }, q_i = \frac{p_i}{p_i - 1}.
    \end{equation}
\end{theorem}
\begin{lemma}\label{lemma:decision boundary distance}
    Let $x \in \R^\ddim$ with $w^Tx + b \neq 0$. Let $p \in [1, \infty]$ and $q$
    such that $\frac{1}{p} + \frac{1}{q} = 1$, where we take $\frac{1}{\infty} =
    0$. Then
    \begin{equation}
        \min \{\|\delta\|_p: \sgn(w^T(x+\delta) + b) \neq \sgn(w^Tx + b) \} = \frac{|w^T + b|}{\|w\|_q}
    \end{equation}
    and the minimum is attained by 
    \begin{equation}
        \delta = 
        \begin{cases} 
        \frac{-w^Tx - b}{\|w\|_\infty} \sgn(w_j) e_j, j = \argmax_i |w_i| & p = 1 \\
        \frac{-w^Tx - b}{\|w\|_q^q}(\sgn(w_i) |w_i|^{\frac{1}{p-1}})_{i=1}^d & p \in (1, \infty]\,.
        \end{cases}\label{eq:defining delta}
    \end{equation}
    where $x^\frac{1}{\infty - 1} = x^0 = 1$ and $e_j$ is the $j$-th unit vector.
\end{lemma}
\begin{proof}[Proof of \cref{thm:robustness curve shape}]
    By Hölder's inequality, for any $\delta$, 
    \begin{equation}
    \sum_{i=1}^\ddim |w_i \delta_i| \leq \|\delta\|_p \|w\|_q \,.
    \end{equation}
    For $\delta$ such that $\sgn(w^T(x+\delta) + b) \neq \sgn(w^Tx + b)$ it
    follows that
    \begin{equation}
    \|\delta\|_p \geq \frac{\sum_{i=1}^\ddim |w_i \delta_i|}{\|w\|_q} \geq \frac{|\sum_{i=1}^\ddim w_i \delta_i|}{\|w\|_q} \geq \frac{|w^Tx + b|}{\|w\|^q} \,. \label{eq:Hölder bound}
    \end{equation}

    Using the identity $q = \frac{p}{p-1}$, it is easy to check that for every
    $p \in [1, \infty]$, with $\delta$ as defined in \cref{eq:defining delta},
    \begin{enumerate}
        \item\label{item:feasible point} $w^T \delta = -w^Tx - b$, so that
        $w^T(x + \delta) + b = 0$, and
        \item\label{item:minimal norm} $\|\delta\|_p = \frac{|w^Tx+b|}{\|w\|_q}$.
    \end{enumerate}
    \Cref{item:feasible point} shows that $\delta$ is a feasible point, while
    \cref{item:minimal norm} in combination with \cref{eq:Hölder bound} shows
    that $\|\delta\|_p$ is minimal.
\end{proof}

Using \cref{lemma:decision boundary distance}, we are ready to prove
\cref{thm:robustness curve shape}.
\begin{proof}
By definition, 
\begin{equation}
    \rob{\ell_{p_1}}{f}(\varepsilon) = \P(\underbrace{\{(x,y) \text{ s.t. } \exists\; \delta: \|\delta\|_{p_1} \leq \varepsilon \wedge f(x+\delta) \neq y\}}_{\rr_{p_1}(\varepsilon)})\,.
\end{equation}
We can split $\rr_{p_1}(\varepsilon)$ into the disjoint sets
\begin{gather}
    \underbrace{\{(x,y): f(x) \neq y\}}_{=M} \\
     \dot\cup \\
     \underbrace{\{(x,y) \;\text{s.t.}\;\exists\; \delta: \|\delta\|_{p_1} \leq \varepsilon \wedge y = f(x) \neq f(x+\delta)\}}_{=B_{p_1}(\varepsilon)}\,.
\end{gather}
Choose $q_1, q_2$ such that $\frac{1}{p_i} + \frac{1}{q_i} = 1$. By
\cref{lemma:decision boundary distance}, and using that $f(x) = \sgn(w^Tx + b)$,
\begin{align}
    B_{p_1}(\varepsilon) &= \{(x,y): \sgn(w^Tx+b) = y \wedge \frac{|w^Tx+b|}{\|w\|_{q_1}} \leq \varepsilon\} \\
    & = \{(x,y): \sgn(w^Tx+b) = y \wedge \frac{|w^Tx+b|}{\|w\|_{q_2}} \leq \frac{\|w\|_{q_1}}{\|w\|_{q_2}}\varepsilon\}) \\
    &= B_{p_2}\left(\frac{\|w\|_{q_1}}{\|w\|_{q_2}}\varepsilon\right)\,.
\end{align}
This shows that
\begin{align}
    \rob{\ell_{p_1}}{f}(\varepsilon) &= \P(M) + \P(B_{p_1}(\varepsilon)) \\
    &= \P(M) + \P\left(B_{p_2}\left(\frac{\|w\|_{q_1}}{\|w\|_{q_2}}\varepsilon\right)\right) \\
    &= \rob{\ell_{p_2}}{f}\left(\frac{\|w\|_{q_1}}{\|w\|_{q_2}}\varepsilon\right)\,.
\end{align}
\end{proof}
\section{Experimental details}\label{app:experimental details}
\subsection{Model training}
We use the same model architecture as \textcite{Croce2018ProvableRobustness} and
\textcite{wong2017provable}. Unless explicitly stated otherwise, the trained
models are taken from \textcite{Croce2018ProvableRobustness}. The exact
architecture of the model is: Convolutional layer (number of filters: 16, size:
4x4, stride: 2), ReLu activation function, convolutional layer (number of
filters: 32, size: 4x4, stride: 2), ReLu activation function, fully connected
layer (number of units: 100), ReLu activation function, output layer (number of
units depends on the number of classes). All models are trained with Adam
Optimizer \autocite{kingma2014adam} for 100 epochs, with batch size 128 and a
default learning rate of 0.001. More information on the training can be found in
the experimental details section of the appendix of
\textcite{Croce2018ProvableRobustness}.
The trained models are those made publicly available by \textcite{Croce2018ProvableRobustness}%
\footnote{\raggedright The models trained with \tmst, \tmkw, \tmat and \tmma are avaible at \url{www.github.com/max-andr/provable-robustness-max-linear-regions}.}%
and \textcite{Croce2020Provable}%
\footnote{\raggedright The models trained with \tmmu are avaible at \url{www.github.com/fra31/mmr-universal}.}%
.
\subsection{Approximated robustness curves}\label{app:adversarial generation details}
We use state-of-the-art adversarial attacks to approximate the true minimal
distances of input datapoints to the decision boundary of a classifier for our
adversarial robustness curves (see \cref{Definition:robustness_curves}). We base
our selection of attacks on the recommendations of
\textcite{carlini2019evaluating}. Specifically, we use the following attacks:
For $\ell_2$ robustness curves we use the $\ell_2$-attack proposed by \textcite{Carlini2017Towards} and for $\ell_\infty$ robustness curves we use PGD~\autocite{madry2018towards}.
For both attacks, we use the implementations of Foolbox \autocite{rauber2017foolbox}. For the
$\ell_\infty$ attack, we use the standard hyperparameters of the Foolbox
implementation. For the $\ell_1$ and $\ell_2$ attacks we increase the number of
binary search steps that are used to find the optimal tradeoff-constant between
distance and confidence from 5 to 10, which we found empirically to improve the
results. For the rest of the hyperparameters, we again use the standard values
of the Foolbox implementation.
\subsection{Computational architecture}
We executed all programs on an architecture with 2 x Intel Xeon(R) CPU E5-2640
v4 @ 2.4 GHz, 2 x Nvidia GeForce GTX 1080 TI 12G and 128 GB RAM.